\documentclass[letterpaper, 10 pt, conference]{ieeeconf}  

\IEEEoverridecommandlockouts                              

\usepackage{graphicx}
\usepackage[hyphens]{url}
\usepackage{amsmath}
\usepackage{amssymb}
\usepackage{booktabs}
\usepackage{algorithm}
\usepackage{algorithmic}
\usepackage{bm}
\usepackage{multirow}
\usepackage{siunitx}
\usepackage{cite}
\usepackage{tikz}
\usetikzlibrary{calc}
\usetikzlibrary{positioning}
\usetikzlibrary{shapes.geometric}
\usetikzlibrary{patterns}
\tikzset{
vertex/.style={circle,draw,black,align=center,inner sep=0cm, minimum size=0.5cm,fill=white,anchor=center},
agent/.style={vertex,fill={rgb:red,1;green,1;blue,1},text=white},
agent1/.style={agent,fill={rgb:red,0;green,1;blue,3}},
agent2/.style={agent,fill={rgb:red,3;green,1;blue,0}},
line/.style={black},
mdd-node/.style={rectangle,draw,black,align=center,inner sep=0.05cm,minimum height=0.4cm,fill=white,anchor=center,thin},
}
\usepackage{macro}
\newtheorem{proposition}{Proposition}
\newtheorem{corollary}{Corollary}

\title{\LARGE \bf
  Iterative Refinement for Real-Time Multi-Robot Path Planning
}

\author{
  Keisuke Okumura$^{1}$,
  Yasumasa Tamura$^{1}$ and
  Xavier D\'{e}fago$^{1}$
  \thanks{$^{1}$The authors are with School of Computing, Tokyo Institute of Technology, Tokyo, Japan.
        {\tt\small \{okumura.k, tamura.y, defago.x\}@coord.c.titech.ac.jp}}%
}

\begin{document}

\maketitle
\thispagestyle{empty}
\pagestyle{empty}

\begin{abstract}
  We study the iterative refinement of path planning for multiple robots, known as multi-agent pathfinding (MAPF).
  Given a graph, agents, their initial locations, and destinations, a solution of MAPF is a set of paths without collisions.
  Iterative refinement for MAPF is desirable for three reasons:
  1)~optimization is intractable,
  2)~sub-optimal solutions can be obtained instantly, and
  3)~it is anytime planning, desired in online scenarios where time for deliberation is limited.
  Despite the high demand, this is under-explored in MAPF because finding good neighborhoods has been unclear so far.
  Our proposal uses a sub-optimal MAPF solver to obtain an initial solution quickly, then iterates the two procedures: 1)~select a subset of agents, 2)~use an optimal MAPF solver to refine paths of selected agents while keeping other paths unchanged.
  Since the optimal solvers are used on small instances of the problem, this scheme yields efficient-enough solutions rapidly while providing high scalability.
  We also present reasonable candidates on how to select a subset of agents.
  Evaluations in various scenarios show that the proposal is promising; the convergence is fast, scalable, and with reasonable quality.
\end{abstract}
\section{Introduction}
Path planning for multiple robots is a fundamental problem in multi-robot coordination.
The objective is to assign each robot on a graph with a collision-free path to its destination.
This problem is known under various names: multi-robot path planning, cooperative path finding, or \emph{multi-agent path finding (MAPF)}~\cite{stern2019def}.
Hereafter, this paper calls it MAPF, where robots are represented as agents moving on a graph.

Applications of MAPF are inherently real-time systems with a limited time for planning,
e.g., automated warehouse~\cite{wurman2008coordinating}, intersection management~\cite{dresner2008multiagent}, airport surface operation~\cite{morris2016planning}, automated parking~\cite{okoso2019multi}, and video games~\cite{silver2005cooperative}.
It is critical to obtain feasible solutions with sufficient quality before deadlines.

Contrary to the importance of efficient path planning, optimization is intractable.
MAPF is known to be an NP-hard problem for various optimization criteria~\cite{yu2013structure}.
This is still true when restricting fields in grid structures~\cite{banfi2017intractability}, or approximating within any constant factor less than 4/3~\cite{ma2016multi}.
Even with state-of-the-art optimal algorithms, planning with hundreds of agents is still challenging~\cite{lam2020new}.

On the other hand, we can create sub-optimal solutions in a very short time~\cite{surynek2009novel,wang2011mapp,luna2011push,de2013push,okumura2019priority}.
Although they often ignore solution quality, having feasible solutions is clearly better than no solution as a result of waiting for optimal solvers, which are not guaranteed to get ones within the deadlines.

With feasible solutions obtained quickly, we can use the remaining time until the deadlines for \emph{iterative refinement}.
This is the basic motivation of \emph{anytime algorithms}~\cite{zilberstein1996using}, which can yield a feasible solution whenever interrupted, the quality of which improving as time passes. 

Iterative refinement is also fruitful for online situations where goals are allocated dynamically~\cite{ma2017lifelong,vsvancara2019online}.
The challenging part here is replanning.
Two intuitive approaches exist: 1)~replan all paths, or 2)~replan a single path for one robot with a new goal while keeping the others unchanged.
The first approach may return efficient solutions but it is costly and typically inappropriate for online use.
The second approach may return inefficient solutions but is nearly costless.
We can apply the second approach for an initial solution that we can gradually refine within the time constraints.

Iterative refinement is promising though under-explored in MAPF because, so far, it was unclear how to incrementally improve a known solution.
In the context of local search, this corresponds to finding a good neighborhood solution.
Hence, we aim to propose an \emph{anytime} framework of \emph{iterative refinement for MAPF}.


\paragraph*{Related Work}
Standley and Korf~\cite{standley2011complete} extended their previous work of optimal MAPF algorithm~\cite{standley2010finding} and developed an anytime version.
Cohen \etal~\cite{cohen2018anytime} studied an anytime algorithm based on a variant of \astar and applied it to MAPF. 
X$^\ast$~\cite{vedder2021x} is an anytime MAPF solver assuming sparse scenarios, i.e., agent distributed sparsely in fields and where potential for collisions is rare.
These three methods search for non-optimal solutions by relaxing some constraints, then eventually converge to optimal solutions by iteratively tightening the constraints.
A drawback is that they are each tied to a specific solver, and that they may fail to obtain initial solutions in a reasonable time thus returning nothing.
Surynek~\cite{surynek2013redundancy} studied local repairing rules for a pebble motion on graphs; which can be adapted to MAPF.
Since improvements are done by ad-hoc local changes, redundancies of \apriori unknown patterns remain in the solution.
%



%

\paragraph*{Contributions}
We propose a generic framework to provide anytime MAPF based on an effective combination of existing solvers.
Our framework first uses a sub-optimal MAPF solver to very quickly obtain an initial feasible solution,
then, it uses an optimal MAPF solver to find good neighborhood solutions.
Precisely, the framework refines the solution iteratively by selecting a subset of agents and using an optimal solver to refine their paths while keeping other paths fixed.
Although the refinement process uses an optimal solver, each refinement are completed quickly because it solves a sub-problem whose size depends on the number of selected agents, typically much smaller than the original.
We also present reasonable candidates on how to select a subset of agents.

We study the effectiveness of the approach in various benchmarks, and observe empirically that the framework converges almost optimally within a short time in small instances, and remains responsive even for very large instances (i.e., large environments and/or many agents). In other words, it brings many practical advantages over prior art.

In a wider view, our study can also be seen as solving a very large-scale neighborhood search~\cite{ahuja2002survey}.
Closer to our concept, Balyo \etal~\cite{balyo2012shortening} studied local replanning for domain-independent planning problems to optimize makespan. It repeats the following; create a sub-problem, obtain an optimal sub-solution by SAT-based techniques, and replace the part of the original solution with a new one.

\paragraph*{Paper Organization}
Section~\ref{sec:preliminaries} provides a formal definition of MAPF and reviews several MAPF solvers that we use.
Section~\ref{sec:ir} describes the framework including basic theoretical analysis.
Section~\ref{sec:selection} presents construction rules of a subset of agents.
Section~\ref{sec:exp} evaluates the proposal in MAPF benchmarks.
Section~\ref{sec:discussion} concludes the paper. 

\section{Preliminaries}
\label{sec:preliminaries}
\subsection{Problem Definition (MAPF)}
The \emph{MAPF problem} is defined as follows.
Consider a set of agents $A = \{a_1, \dots, a_n\}$ evolving in an environment represented as a connected and undirected graph $G = (V, E)$.
Let $\loc{i}{t} \in V$ denote the location of an agent $a_i$ at discrete time~$t \in \mathbb{N}$.
At each timestep $t$, $a_i$ can move to an adjacent node, or stay at its current location, i.e., $\loc{i}{t+1} \in \{ v\;|\;(\loc{i}{t}, v) \in E\} \cup \{ \loc{i}{t} \}$.
Agents must avoid two types of conflicts: $\loc{i}{t} \neq \loc{j}{t}$ and $\loc{i}{t} \neq \loc{j}{t+1} \lor \loc{i}{t+1} \neq \loc{j}{t}$.
Given a distinct initial location \loc{i}{0} and a distinct goal $g_i \in V$ for each agent $a_i$, a solution $\paths = (\path{1}, \dots, \path{n})$ is a collection of paths (one for each agent) where $\path{i} = (\loc{i}{0}, \loc{i}{1}, \dots, \loc{i}{T})$ such that $\loc{i}{T} = g_i$.

To evaluate solution quality, we use the \emph{sum-of-costs} metric: $\sum_{a_i \in A}T_i$ where $T_i$ is the minimum timestep such that $\loc{i}{T_i}=\loc{i}{T_i+1}=\ldots=\loc{i}{T}$.
This is a commonly-used objective in MAPF studies~\cite{stern2019def}.
Whenever obvious from context, we simply refer to sum-of-costs as ``cost.''

We further use the following terms.
\dist{u}{v} is the shortest path length between two nodes $u, v \in V$.
\costone{i} is the cost of an agent $a_i$ in a solution \paths, i.e., $T_i$.

\subsection{MAPF Solvers}
This part explains several MAPF solvers that we will use later.
Numerous solvers have been developed so far, and they
can be categorized as: optimal or sub-optimal; complete or incomplete; search-based, prioritized planning, or rule-based.
See~\cite{felner2017search,stern2019def} for comprehensive reviews.
In our experiments, we used ECBS, (W)\hca, RPP, PS, and PIBT to obtain initial sub-optimal solutions and then used ICBS to refine solutions.
We also used CBSH and AFS as a comparison.

Conflict-based Search (CBS)~\cite{sharon2015conflict}, a popular optimal and complete MAPF solver, uses a two-level search.
The high-level search manages conflicts between agents.
When a conflict occurs between two agents at some time and location, two possible resolutions are depending on which agent gets to use the location at that time.
Following this observation, CBS constructs a binary tree where each node includes constraints prohibiting to use space-time pairs for certain agents.
In the low-level search, agents find a single path that is constrained by the corresponding high-level node.

Many studies enhance CBS.
Improved-CBS (ICBS)~\cite{boyarski2015icbs} prioritizes conflicts when splitting a high-level node with several conflicts.
CBSH~\cite{felner2018adding} adds admissible heuristic for high-level nodes.
Improved heuristics have been proposed~\cite{li2019improved}.

Enhanced CBS (ECBS)~\cite{barer2014suboptimal}, a variant of CBS, is a complete and bounded suboptimal solver, i.e., a returned solution is within a given sub-optimality bound.
Instead of best-first search, ECBS uses focal search in both high- and low-level searches.
Focal search~\cite{pearl1982studies}, a variant of \astar, allows exploring efficient nodes not belonging to optimal solutions, e.g., in CBS, high-level nodes with few conflicts.

Anytime Focal Search (AFS)~\cite{cohen2018anytime}, an anytime version of the focal search, iteratively refines a solution with guaranteed solution quality.
Given enough time, AFS finally converges to an optimal solution.
Cohen \etal~\cite{cohen2018anytime} applied AFS to the high-level search of CBS and realized an anytime MAPF solver.

Different from \emph{search-based} approaches explained so far, Hierarchical Cooperative \astar (\hca)~\cite{silver2005cooperative}, neither complete nor optimal, takes a decoupled approach.
\hca is a typical example of \emph{prioritized planning}, i.e., it plans paths for agents sequentially while avoiding conflicts with previously planned paths.
During the construction of prioritized paths, \hca uses the shortest path length between two locations ignoring collisions as heuristics.
Windowed \hca (\whca)~\cite{silver2005cooperative} is a variant of \hca, which uses a limited lookahead window.
In general, prioritized planning is fast, scalable, and practical with acceptable costs (e.g., \cite{silver2005cooperative,vcap2015prioritized,ma2019searching}).

{\v{C}}{\'a}p \etal~\cite{vcap2015prioritized} analyzed a sufficient condition that sequential conflict-free paths are constructed, and proposed Revisit Prioritized Planning (RPP) that agents plan paths while avoiding initial locations of all lower priority agents.
RPP is complete for \emph{well-formed} instances; for each pair of start and goal, a path exists that traverses no other starts and goals.
Note that well-informed instances are hard to realize in dense scenarios.

Push and Swap/Rotate (PS)~\cite{luna2011push,de2013push} is complete, sub-optimal, and an example of \emph{rule-based} approaches.
PS relies on two primitives: \emph{push} to move an agent towards its goal, and \emph{swap} to allow two agents to swap their locations without altering the configuration of other agents.
It only allows a single agent or a pair of agents to move at each time.
In general, rule-based approaches (e.g., \cite{luna2011push,surynek2009novel,wang2011mapp,de2013push}) are the fastest class to obtain feasible solutions but their quality is overlooked.

Priority Inheritance with Backtracking (PIBT)~\cite{okumura2019priority}, incorporating both prioritized planning and rule-based, plans locations of all agents only for the next timestep and repeats this to obtain sub-optimal solutions.
It ensures that all agents reach their goals, but agents might not be on their goals simultaneously; hence it is incomplete for one-shot MAPF.

\section{Iterative Refinement}
\label{sec:ir}
The framework first obtains an initial solution by a sub-optimal MAPF solver, and then iteratively refines selected parts of the solutions, the paths of a selected subset of the agents,  using an optimal MAPF solver.
We show the pseudocode in Algorithm~\ref{algo:framework}.
{
  \begin{algorithm}
    \small
    \caption{The Framework of Iterative Refinement}
    \label{algo:framework}
    \textbf{Input}:~MAPF instance\\
    \textbf{Output}:~solution $\paths$ or $\textbf{FAILURE}$\\
    \vspace{-0.4cm}
    \begin{algorithmic}[1]
      \STATE $\paths \leftarrow$ initial solution obtained by an MAPF solver
      \label{algo:framework:init}
      \IFSINGLE{failed to obtain \paths}{\textbf{return} \textbf{FAILURE}}
      \label{algo:framework:failure}
      \WHILE{not interrupted}
      \label{algo:framework:while}
      \STATE Create a modification set $M \subseteq A$ using \paths
      \label{algo:framework:modif}
      \STATE $\paths \leftarrow \begin{aligned}[t]
        &\text{refined MAPF solution for $M$}
        \\[-0.1cm]
        &\text{while fixing the others' paths in \paths}
      \end{aligned}$
      \label{algo:framework:refine}
      \ENDWHILE
      \label{algo:framework:end-while}
      \RETURN \paths
      \label{algo:framework:return}
    \end{algorithmic}
  \end{algorithm}
}

An initial feasible solution is quickly obtained by a sub-optimal solver~[Line~\ref{algo:framework:init}].
We refer to the used sub-optimal MAPF solver as an \emph{initial solver}.
If the initial solver failed to obtain solutions, the framework ends with a failure~[Line~\ref{algo:framework:failure}];
otherwise, the refinement starts~[Lines~\ref{algo:framework:while}--\ref{algo:framework:end-while}].
The refinement iteratively follows two procedures until interrupted:
1)~Create a modification set $M \subseteq A$ using the current solution \paths~[Line~\ref{algo:framework:modif}].
2)~Refine the current solution \paths by changing paths for agents in $M$~[Line~\ref{algo:framework:refine}] using an optimal MAPF solver.
We call this solver a \emph{refinement solver}.
The refinement solver only changes the paths for agents in $M$; paths for agents not in $M$ are unchanged.
The refinement continues until interrupted, e.g., timeout, reaching the predetermined iteration number, when no improvement is expected, interruption by users, etc.
The framework eventually returns the final solution~[Line~\ref{algo:framework:return}].

The initial solver can be any sub-optimal MAPF solver, as long as it provides feasible solutions.
As the refinement solver, it is desirable to use versions adapted from an optimal solver.
The adaptation is simple; let it plan paths for agents in $M$ regarding the others as dynamic obstacles.
E.g., for CBS, solve MAPF only for agents in $M$ while prohibiting the low-level search to use all space-time pairs used by agents outside of $M$.
In a precise sense, the refinement solver is not limited to optimal MAPF solvers.
The requirement is that the refined solution never worsens from the original.
Considering that cost of paths for agents outside of $M$ does not change, the requirement is that cost of paths for agents in $M$ is non-increasing before and after refinement.

The next property is obvious from the requirements set on the refinement solver.
\begin{corollary}[Monotonicity]
  For each iteration in Algorithm~\ref{algo:framework}, the solution cost is non-increasing.
  \label{cor:monotonicity}
\end{corollary}

A key point is that the refinement solver recomputes the paths for a selected subset $M$ of agents, rather than for the entire set $A$ of all agents.
Compared to solving the original problem directly using optimal solvers, the problem solved at each iteration by the refinement solver is significantly smaller, ensuring that the framework is scalable even to a large number of agents.

\subsection{Early Stop}
\label{subsec:early-stop}
Even though sub-problems solved by the refinement solver are small compared to the original problem, the refinement may still take too long if $|M|$ is too big.
In such cases, it is preferable to abort the current refinement by returning the current solution, and then start a new iteration with a new set $M$.
The criteria can be a timeout or using a threshold value for the size of a search tree in the refinement solver.

\subsection{Limitations}
As a limitation, the framework may have the local minimum with no sub-optimality bounds from the optimal.
\begin{proposition}[No sub-optimality bounds]
  Consider the optimal cost $c^\ast$.
  In Algorithm~\ref{algo:framework}, there is no $w \geq 1$ such that always $c \leq wc^\ast$ unless selecting $A$ itself as a modification set $M$, where $c$ is the solution cost in each iteration.
  \label{prop:no-bound}
\end{proposition}
\begin{proof}
  Consider an example in Fig.~\ref{fig:local-minimum}.
  Assume that an initial solution assigns $a_1$ to a clockwise path (cost: $k$) and $a_2$ to a counterclockwise path (cost: $1$).
  With $k \geq 6$, this is not optimal because $a_1$ can take a counterclockwise path if $a_2$ temporally moves over from its goal (sum-of-costs: $6$).
  Unless $M \neq A$, the solution of the refinement is unchanged.
  Assume $w \geq 1$ such that $c \leq wc^\ast$.
  We can take an arbitrary $k$, contradicting the existence of $w$.
\end{proof}

\begin{figure}[ht!]
  \centering
  \begin{tikzpicture}
    \scriptsize
    \node[vertex](v1) at (0.6, 0.0) {};
    \node[vertex](v2) at (0.0, 0.2) {};
    \node[agent2](v3) at (0.0, 1.0) {$a_2$};
    \node[agent1](v4) at (0.6, 1.2) {$a_1$};
    \node[vertex](v5) at (-0.6, 0.0) {};
    \foreach \u / \v in {v1/v2,v2/v3,v3/v4,v5/v2}
    \draw[line] (\u) -- (\v);
    \node[label=right:{length: $k$}](label-len) at (2.0, 0.6) {};
    \draw[line, dashed] (v1.east) to[out=0,in=270] (label-len) to[out=90,in=0] (v4.east);
    \draw[->,very thick] (v3) to[out=240,in=120] (v2);
    \draw[->,very thick] (v4) -- (v1);
  \end{tikzpicture}
  \vspace{-0.2cm}
  \caption{
    \textbf{Example of a local minimum.}
    Goals are depicted by arrows.
  }
  \label{fig:local-minimum}
\end{figure}
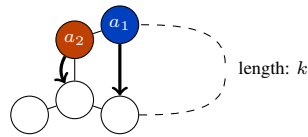

\begin{corollary}[Existence of local minimum]
  Depending on initial solutions, it may be impossible to reach the optimal solution unless selecting $A$ itself as $M$.
  \label{prop:local-minimum}
\end{corollary}

Note that when $M = A$ the refinement solver has to solve the original MAPF problem.

\section{Design of Modification Set}
\label{sec:selection}
The modification set is an important component of the framework, and its design will affect the performance such as computation time and solution quality.
This section defines several \emph{selection rules} to provide reasonable candidates.

\subsection{Random}
\label{subsec:random}
One naive approach is to pick a subset of agents randomly.
The size of a modification set $M$ is then a user-specified parameter.
Note that large $|M|$ has a chance to reduce costs largely in one iteration but takes time for the refinement because sub-problems become challenging.

\subsection{Single Agent}
\label{subsec:single-agent}
This rule always picks a single agent as $M$;\footnotemark can be regarded as a special case of the previous rule (\textit{random}).
Even with a single agent, the cost might be reduced by the refinement.
In this case, the refinement becomes just a single-agent path finding problem and can be computed efficiently without MAPF solvers, e.g., by \astar.

\subsection{Focusing at Goals}
\label{subsec:focus-at-goal}
Consider an example in Fig.~\ref{fig:local-repair}.
Assume that the current solution is $\path{1}=(v_2, v_3, \underline{v_6}, v_3, v_3)$ and $\path{2}=(v_1, v_2, \underline{v_3}, v_4, v_5)$.
An agent $a_1$ cannot achieve a shorter path because an agent $a_2$ uses a goal of $a_1$ (i.e., $v_3$) at a timestep~$2$.
In general, for $a_i$, one reason of a gap between ideal cost \distone{i} and real cost \costone{i} may be that another agent $a_j$ uses a goal for $a_i$ (i.e., $g_i$) at a timestep $t \geq \distone{i}$.
At least before $t$, $a_i$ cannot arrive at $g_i$ and remain there.
In this case, it is required to jointly refine paths of $a_i$ and $a_j$.

This observation motivates to create a following simple rule taking a current solution $\paths$ and one agent $a_i$ as input.
\begin{align*}
  M \leftarrow \{ a_j | \loc{j}{t} = g_i, \distone{i} \leq t \leq \costone{i} \}
\end{align*}
The selecting rule of $a_i$ is arbitrary.\footnotemark[\value{footnote}]

{
  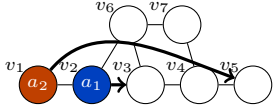
\begin{figure}[ht!]
    \centering
    \begin{tikzpicture}
      \scriptsize
      {
        \newcommand{\edgesize}{0.2cm}
        \newcommand{\labelpos}{-0.15cm}
        \node[agent2](v1) at (0, 0) {$a_2$};
        \node[agent1,right=\edgesize of v1](v2) {$a_1$};
        \node[vertex,right=\edgesize of v2](v3) {};
        \node[vertex,right=\edgesize of v3](v4) {};
        \node[vertex,right=\edgesize of v4](v5) {};
        \node[vertex](v6) at (1.2,0.8) {};
        \node[vertex,right=\edgesize of v6](v7) {};
        \node[above left=\labelpos of v1]{$v_1$};
        \node[above left=\labelpos of v2]{$v_2$};
        \node[above left=\labelpos of v3]{$v_3$};
        \node[above left=\labelpos of v4]{$v_4$};
        \node[above left=\labelpos of v5]{$v_5$};
        \node[above left=\labelpos of v6]{$v_6$};
        \node[above left=\labelpos of v7]{$v_7$};

        \foreach \u / \v in {v1/v2,v2/v3,v3/v4,v4/v5,v2/v6,v6/v7,v7/v4,v3/v6}
        \draw[line] (\u) -- (\v);
        \draw[->,very thick] (v2) to[out=0,in=180] (v3);
        \draw[->,very thick] (v1) to[out=50,in=160] (v5);
      }
    \end{tikzpicture}
    \vspace{-0.2cm}
    \caption{
      \textbf{Example of local repiar around goals.}
    }
    \label{fig:local-repair}
  \end{figure}
}

\subsection{Local Repair around Goals}
\label{subsec:local-repair}
This is a special case of the previous rule (\textit{focusing-at-goals}).
Assume again the example in Fig.~\ref{fig:local-repair}; $\path{1}=(v_2, v_3, \underline{v_6}, v_3, v_3)$ and $\path{2}=(v_1, v_2, v_3, v_4, v_5)$.
In \textit{focusing-at-goals}, $M$ for $a_1$ is $\{ a_1, a_2 \}$, therefore, the refinement solver has to solve a sub-problem with two agents; however, this effort can be reduced.
Consider obtaining a better path for $a_1$ ignoring \path{2}.
In this example, a new path is obtained without searching by \emph{local repair around the goal}; $(v_2, v_3, \underline{v_3}, v_3, v_3)$.
Next, compute a single path for $a_2$ while avoiding collisions with this new path and the other agents' paths.
If the sum of costs of two new paths is smaller than the original, replace \path{1} and \path{2} with the new paths.
Since search effort is reduced, the refinement is expected to finish faster.

In general, when $\path{i}=(\ldots, g_i, v, g_i, \ldots, g_i)$ where $v \neq g_i$ and another agent $a_j$ uses $g_i$ at that timestep, this rule can be applied.

\subsection{Using MDD}
\label{subsec:using-mdd}
Given a single path cost $c$, a set of paths from \loc{i}{0} to $g_i$ can be compactly represented as a \emph{multi-valued decision diagram (MDD)}~\cite{srinivasan1990algorithms}; a directed acyclic graph where a vertex is a pair of a location $v \in V$ and a timestep $t \in \mathbb{N}$.
Each vertex in an MDD satisfies two conditions:
1)~a reachable location at that timestep from a start and
2)~a reachable location to a goal from that timestep.
Let \mdd{c}{i} be an MDD for $a_i$ with a cost $c$.
Fig.~\ref{fig:mdd} shows two examples: \mdd{2}{1} and \mdd{3}{1}.
MDDs are used often in MAPF solvers~\cite{sharon2013increasing,boyarski2015icbs}.

{
  \newcommand{\mddedgesize}{0.65cm}
  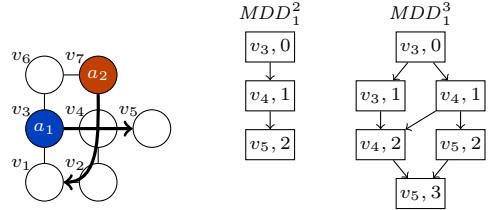
\begin{figure}[t]
    \centering
    \begin{tikzpicture}
      \scriptsize
      {
        \newcommand{\edgesize}{0.2cm}
        \newcommand{\labelpos}{-0.15cm}
        \node[vertex](v1) at (0, 0) {};
        \node[vertex,right=\edgesize of v1](v2) {};
        \node[agent1,above=\edgesize of v1](v4) {$a_1$};
        \node[vertex,right=\edgesize of v4](v5) {};
        \node[vertex,right=\edgesize of v5](v6) {};
        \node[vertex,above=\edgesize of v4](v7) {};
        \node[agent2,right=\edgesize of v7](v8) {$a_2$};
        \node[above left=\labelpos of v1]{$v_1$};
        \node[above left=\labelpos of v2]{$v_2$};
        \node[above left=\labelpos of v4]{$v_3$};
        \node[above left=\labelpos of v5]{$v_4$};
        \node[above left=\labelpos of v6]{$v_5$};
        \node[above left=\labelpos of v7]{$v_6$};
        \node[above left=\labelpos of v8]{$v_7$};
        \foreach \u / \v in {v1/v2,v4/v5,v5/v6,v7/v8,v1/v4,v2/v5,v4/v7,v5/v8}
        \draw[line] (\u) -- (\v);
        \draw[->,very thick] (v4) to[out=0,in=180] (v6);
        \draw[->,very thick] (v8) to[out=270,in=0] (v1);
      }
      {
        \coordinate[](t0) at (3, 1.8);
        \coordinate[below=\mddedgesize of t0](t1);
        \coordinate[below=\mddedgesize of t1](t2);
        \node[mdd-node,label=above:{\mdd{2}{1}}](mdd2_3_0) at (t0) {$v_3, 0$};
        \node[mdd-node](mdd2_4_1) at (t1) {$v_4, 1$};
        \node[mdd-node](mdd2_5_2) at (t2) {$v_5, 2$};
        \foreach \u / \v in {mdd2_3_0/mdd2_4_1,mdd2_4_1/mdd2_5_2}
        \draw[line,->] (\u) -- (\v);
      }
      {
        \coordinate[](t0) at (5, 1.8);
        \coordinate[below=\mddedgesize of t0](t1);
        \coordinate[below=\mddedgesize of t1](t2);
        \coordinate[below=\mddedgesize of t2](t3);
        \node[mdd-node,label=above:{\mdd{3}{1}}](mdd3_3_0) at (t0) {$v_3, 0$};
        \node[mdd-node,left=0.2cm of t1](mdd3_3_1) {$v_3, 1$};
        \node[mdd-node,right=0.2cm of t1](mdd3_4_1) {$v_4, 1$};
        \node[mdd-node,left=0.2cm of t2](mdd3_4_2) {$v_4, 2$};
        \node[mdd-node,right=0.2cm of t2](mdd3_5_2) {$v_5, 2$};
        \node[mdd-node](mdd3_5_3) at (t3) {$v_5,3$};
        \foreach \u / \v in {mdd3_3_0/mdd3_4_1,mdd3_3_0/mdd3_3_1,mdd3_4_1/mdd3_5_2,
          mdd3_3_1/mdd3_4_2,mdd3_4_1/mdd3_4_2,mdd3_5_2/mdd3_5_3,mdd3_4_2/mdd3_5_3}
        \draw[line,->] (\u) -- (\v);
      }
    \end{tikzpicture}
    \vspace{-0.2cm}
    \caption{
      \textbf{Examples of MDD.}
    }
    \label{fig:mdd}
  \end{figure}
}

Using \mdd{c}{i} where $\distone{i} \leq c < \costone{i}$, a set of agents interfering with \path{i} can be detected.
See an example in Fig.~\ref{fig:mdd}.
Assume that the current solution is $\path{1}=(v_3, v_3, v_4, v_5)$ and $\path{2}=(v_7, v_4, v_2, v_1)$.
Consider to update \mdd{2}{1} by \path{2};
remove vertices of \mdd{2}{1} that collides of \path{2}, i.e., $(v_4, 1)$.
Then, remove all redundant vertices that do not satisfy the two conditions due to the first operation: ($v_3, 0$) and ($v_5, 2$).
Now it turns out to be impossible that $a_1$ reaches its goal with a cost of $2$ because there is no remaining vertex.
In other words, \path{2} is preventing that $a_1$ achieves a smaller cost; hence \path{1} and \path{2} should be jointly refined.
We describe the general procedure in Algorithm~\ref{algo:identify-by-mdd}.\footnotemark[\value{footnote}]

{
  \begin{algorithm}
    \small
    \caption{\textit{using-MDD}}
    \label{algo:identify-by-mdd}
    \textbf{Input}:~current solution \paths, an agent $a_i \in A$\\
    \textbf{Output}:~modification set $M \subseteq A$\\
    \vspace{-0.4cm}
    \begin{algorithmic}[1]
      \STATE $M \leftarrow \{ a_i \}$
      \FOR{$\distone{i} \leq c < \costone{i}$}
      \STATE create \mdd{c}{i}
      \FOR{$a_j \in A \setminus \{ a_i \}$}
      \STATE update \mdd{c}{i} by \path{j}
      \IFSINGLE{\mdd{c}{i} is updated by \path{j}}{$M\leftarrow M \cup \{ a_j \}$}
      \ENDFOR
      \ENDFOR
      \RETURN $M$
    \end{algorithmic}
  \end{algorithm}
}

\subsection{Using Bottleneck Agent}
\label{subsec:using-bottleneck}
Consider the example of Fig.~\ref{fig:mdd} again; $\path{1}=(v_3, v_3, v_4, v_5)$ and $\path{2}=(v_7, v_4, v_2, v_1)$.
If removing \path{2}, $a_1$ can take a shorter path, meaning that, $a_2$ is a \emph{bottleneck} for $a_1$.
There is a chance to reduce a cost by refining jointly with such a bottleneck agent and agents that can take shorter paths without the agent.
We describe this concept in Algorithm~\ref{algo:identify-by-bottleneck}.\footnotemark[\value{footnote}]
{
  \begin{algorithm}
    \small
    \caption{\textit{using-bottleneck-agent}}
    \label{algo:identify-by-bottleneck}
    \textbf{Input}:~current solution \paths, an agent $a_i \in A$\\
    \textbf{Output}:~modification set $M \subseteq A$\\
    \vspace{-0.4cm}
    \begin{algorithmic}[1]
      \STATE $M \leftarrow \{ a_i \}$
      \FOR{$a_j \in A \setminus \{ a_i \}$}
      \STATE $c \leftarrow \begin{aligned}[t]
        &\text{cost of the best path for}~a_j\\[-0.1cm]
        &\text{while avoiding collisions with}~\paths \setminus \{ \path{i}, \path{j}\}
      \end{aligned}$
      \IFSINGLE{$c < \costone{j}$}{$M\leftarrow M \cup \{a_j\}$}
      \ENDFOR
      \RETURN $M$
    \end{algorithmic}
  \end{algorithm}
}

\footnotetext{In our implementation, an agent $a_i$ is selected sequentially.}

\subsection{Composition}
\label{subsec:composition}
Each rule might have suitable situations, e.g., the rule \textit{focusing-at-goals} (Sec.~\ref{subsec:focus-at-goal}) is costless to create modification sets, but it might be weak to detect effective sets when solutions are already efficient to some extent.
On the other hand, the rule \textit{using-MDD} (\ref{subsec:using-mdd}) takes time but they are highly expected to detect effective sets.
Therefore, one promising direction is to \emph{composite} these rules, namely, execute the first rule until no improvement is expected, and then switch to the second rule; same as above.

{
  \newcommand{\field}[4]{%
    \begin{minipage}{0.28\hsize}%
      \centering%
      \begin{tabular}{ll}
        \begin{minipage}{0.5\hsize}
          \includegraphics[width=1\hsize]{fig/raw/map/#1.pdf}
        \end{minipage}
        \begin{minipage}{0.6\hsize}
          {\tiny\fieldname{#1}\vspace{-0.2cm}\\$#2\stimes#3$\vspace{-0.2cm}\\(#4)}
        \end{minipage}
      \end{tabular}
    \end{minipage}%
  }
  \begin{figure}[t]
    \centering
    \begin{tabular}{ccc}
      \field{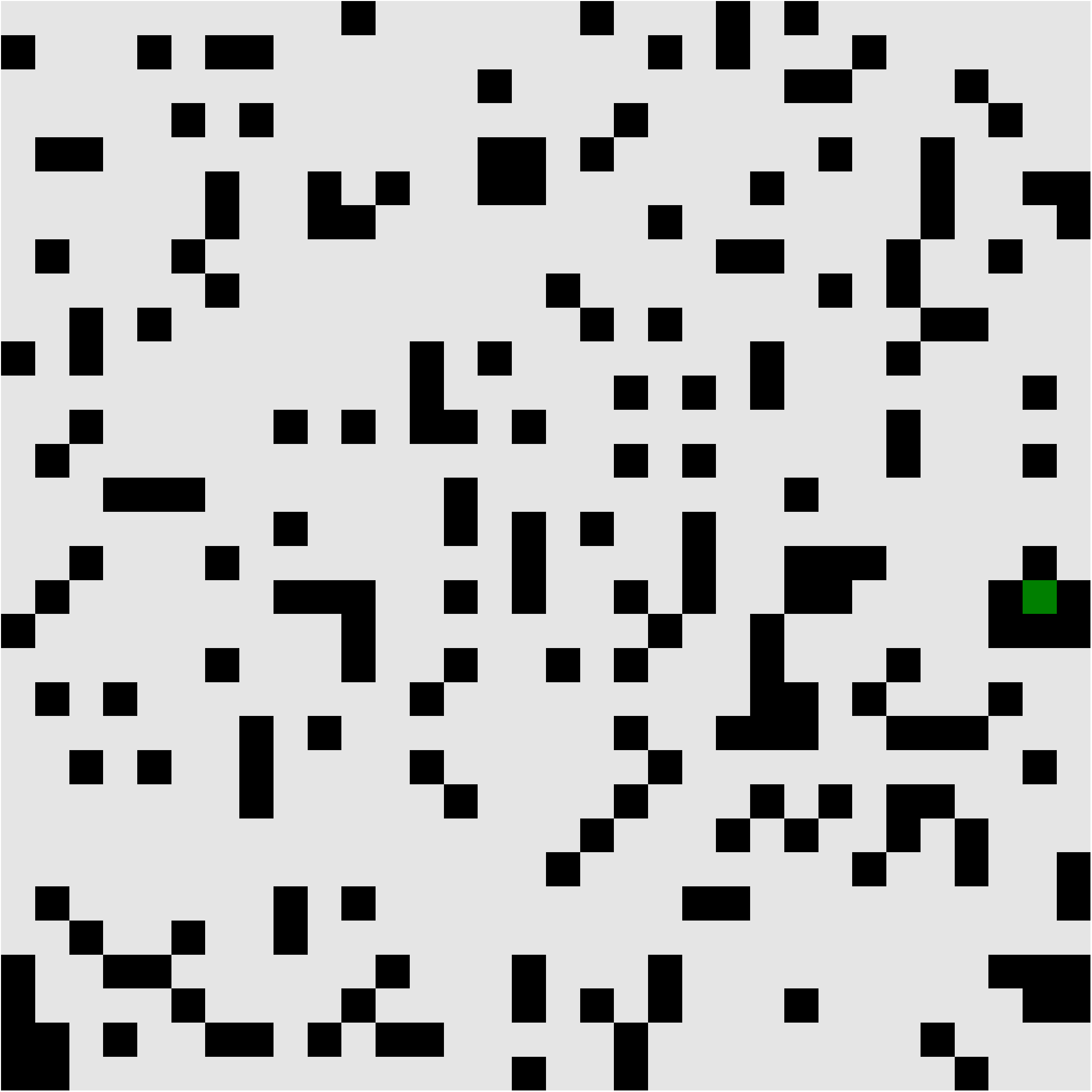}{32}{32}{819}
      &
        \field{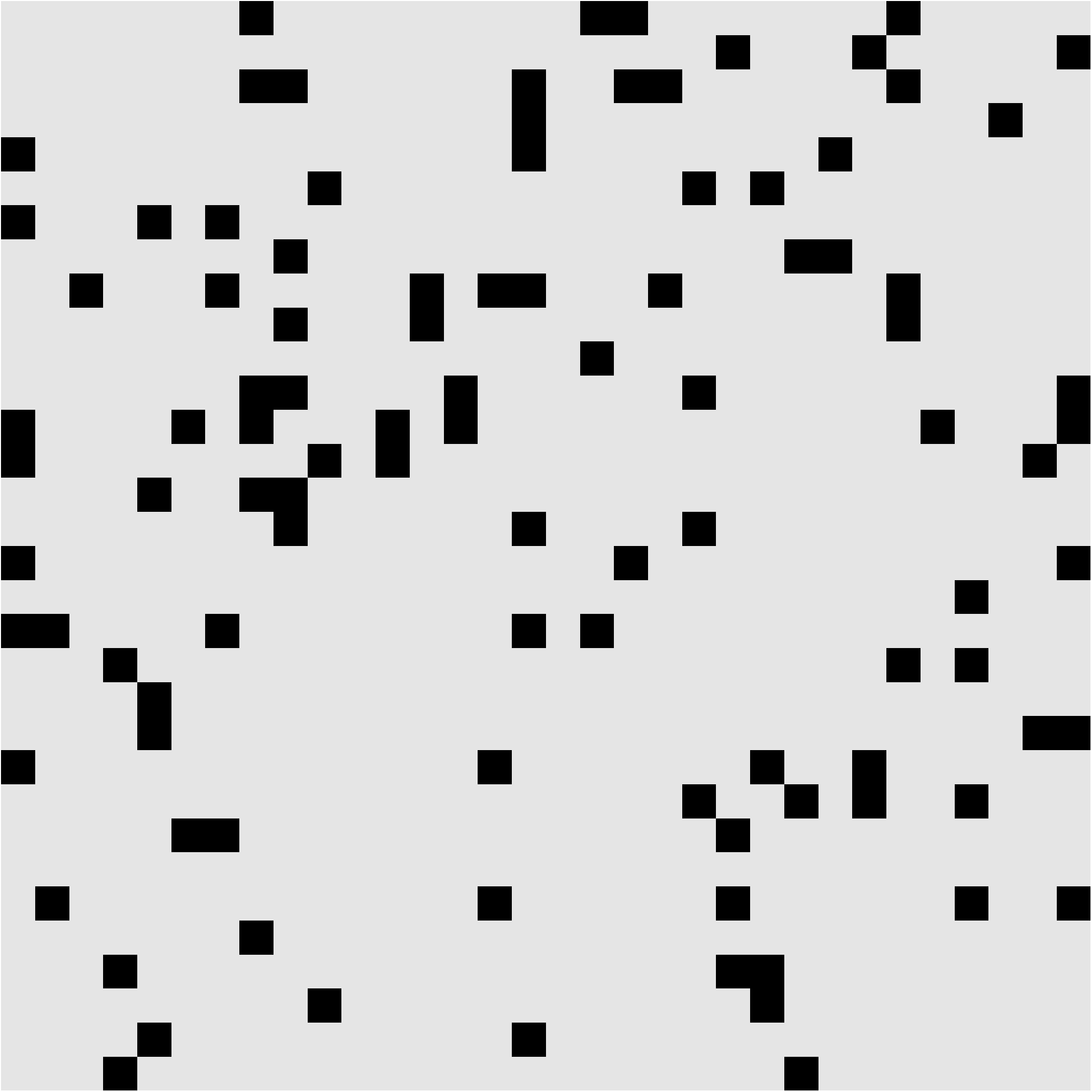}{32}{32}{922}
      &
        \field{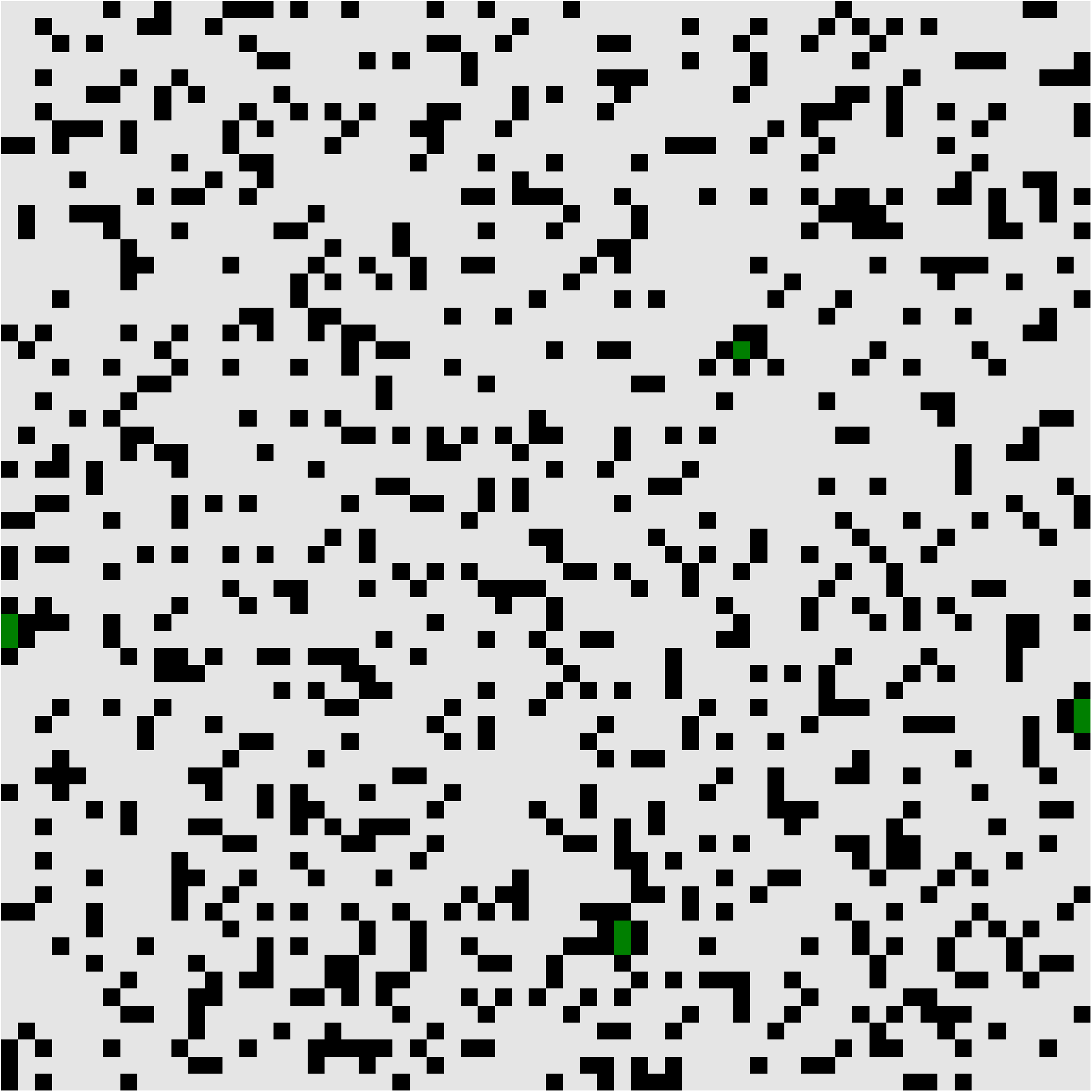}{64}{64}{3,270}
        \vspace{0.1cm}
      \\
      \field{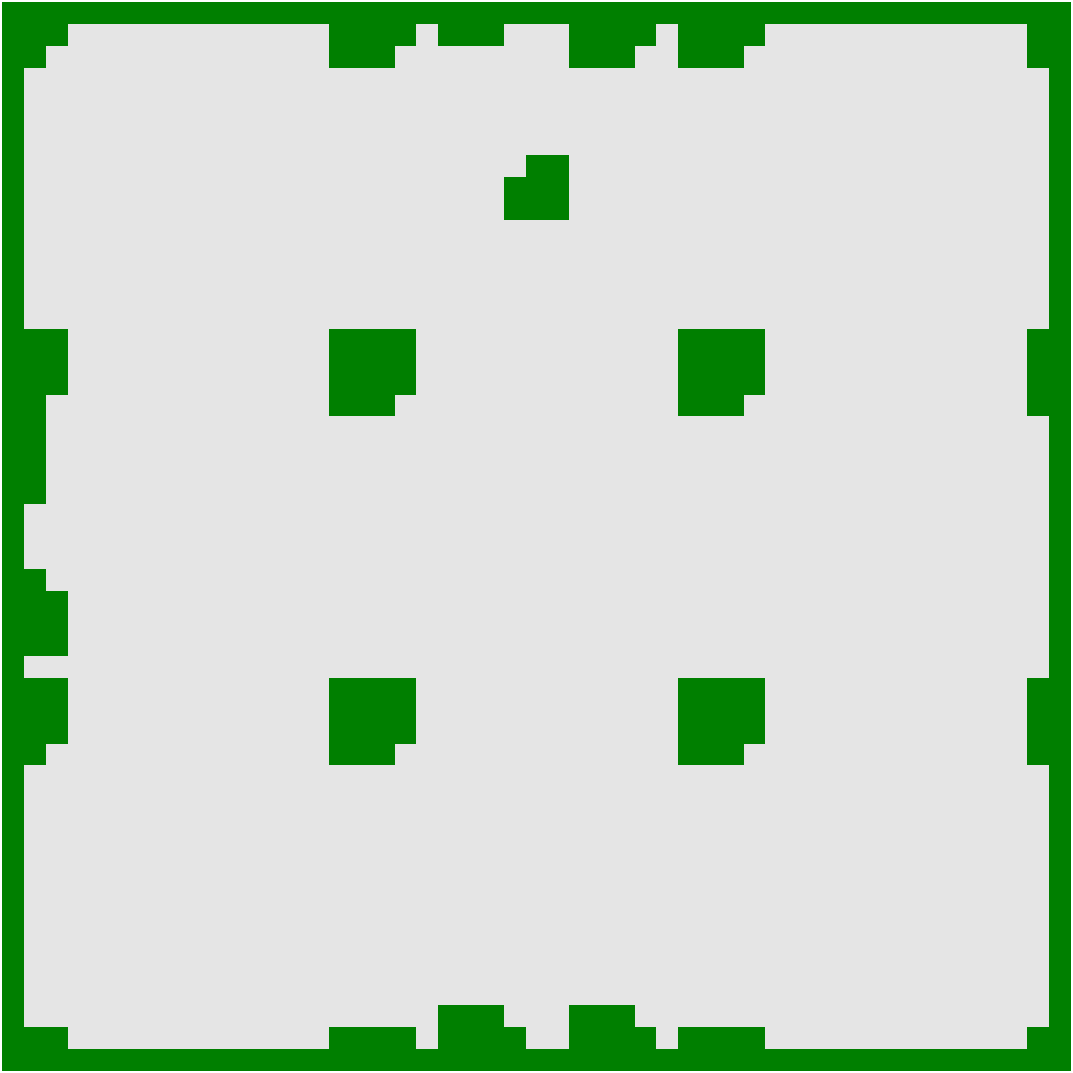}{49}{49}{2054}
      &
        \field{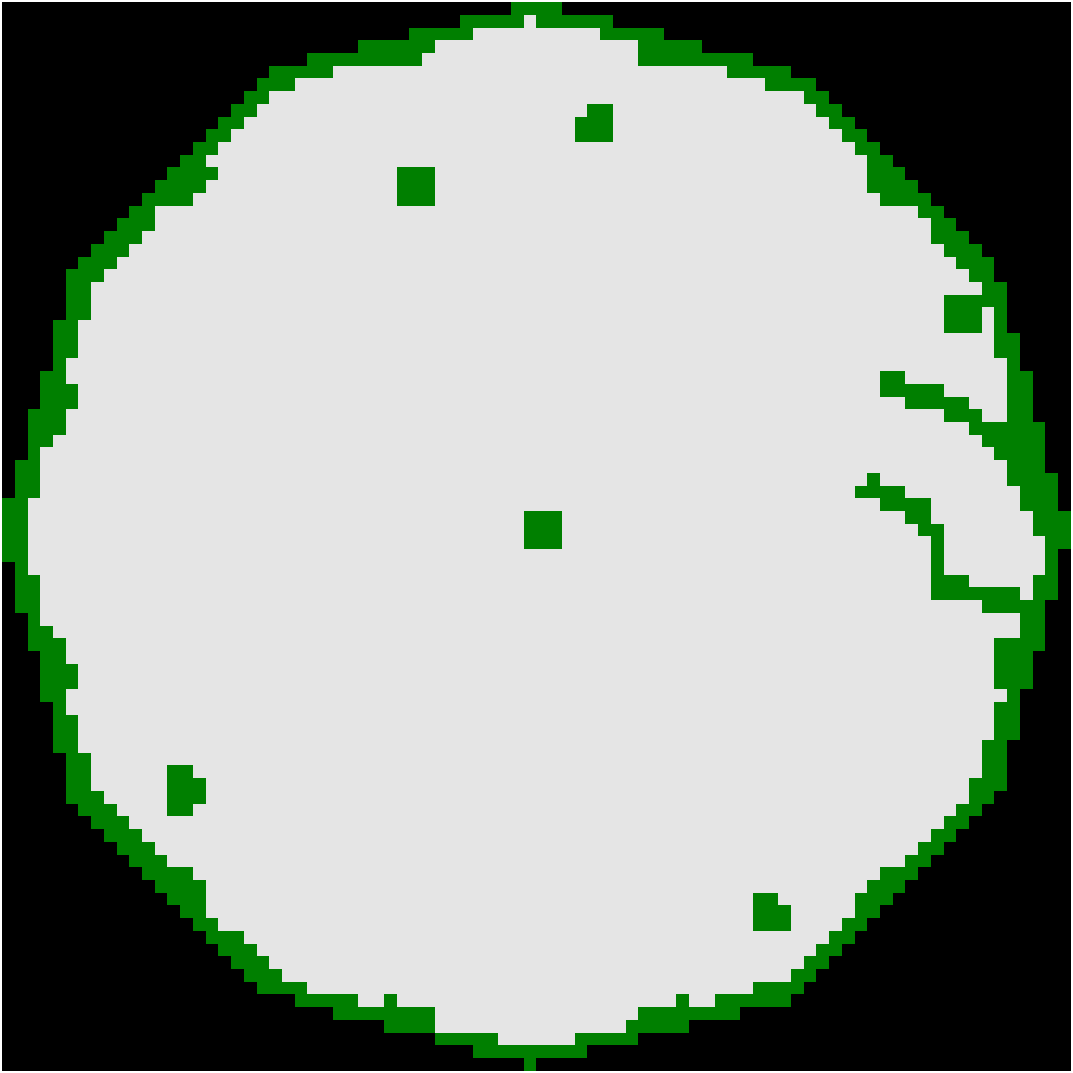}{84}{84}{4,706}
      &
        \field{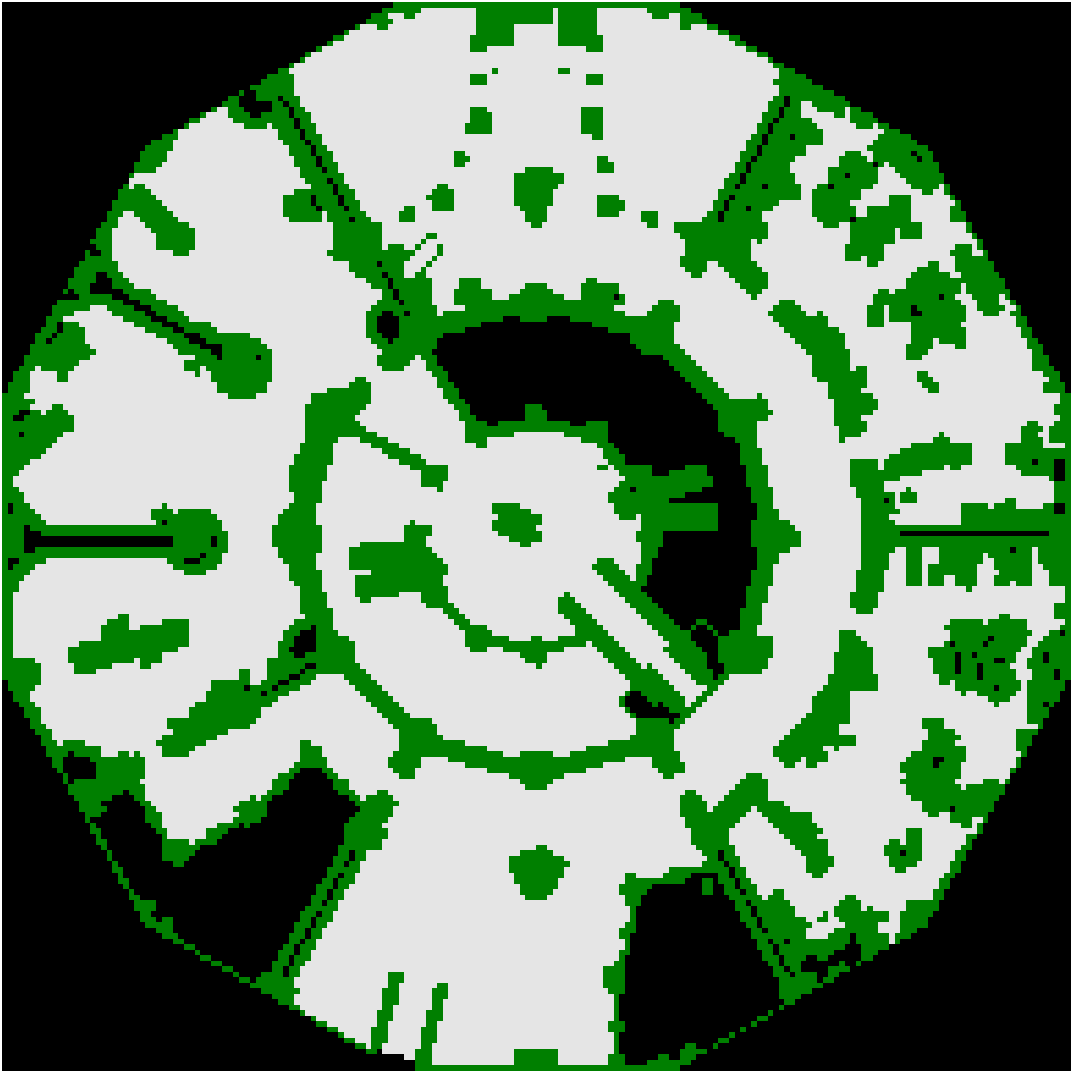}{194}{194}{17,953}
        \vspace{0.1cm}
      \\
      \field{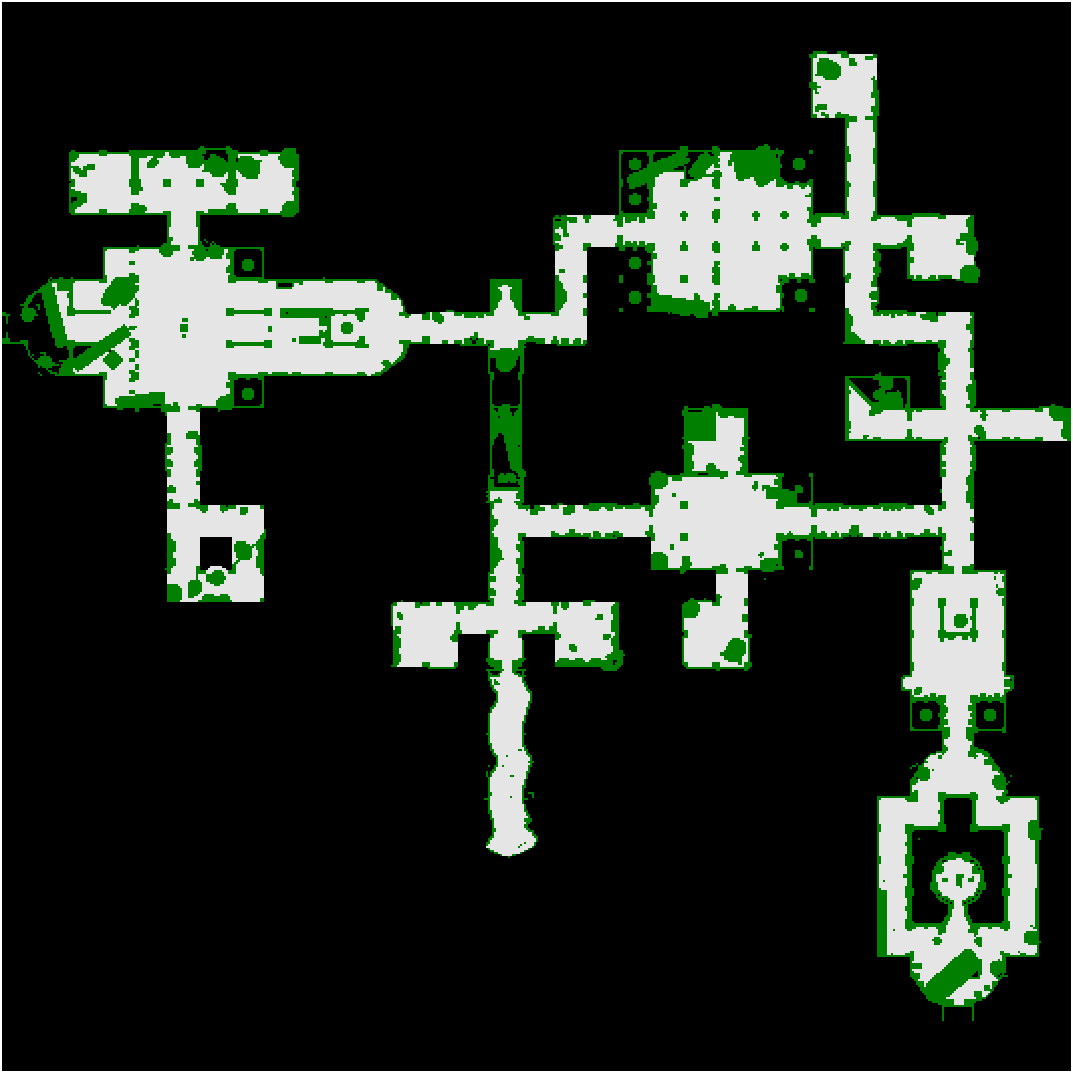}{481}{530}{43,151}&
      \field{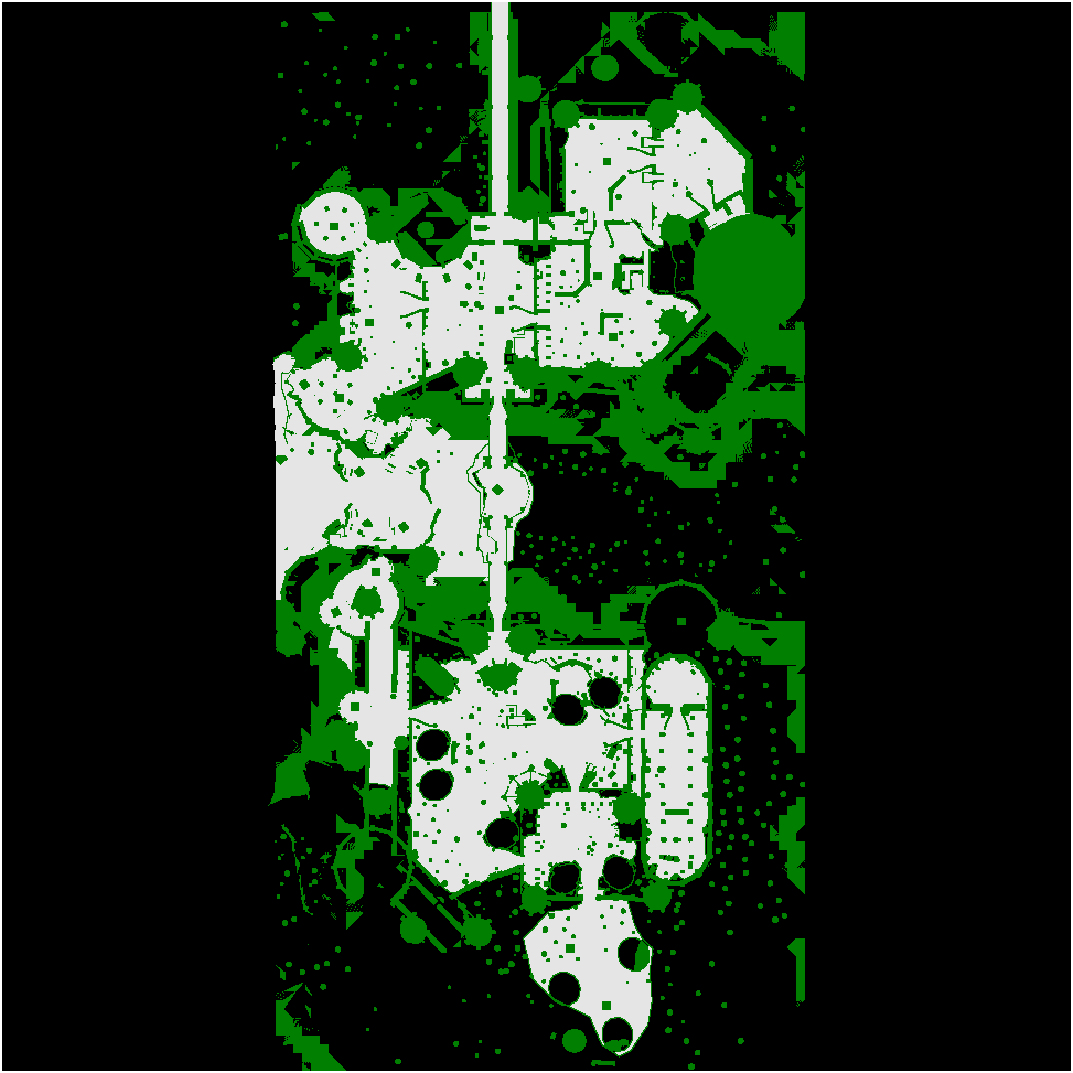}{969}{487}{130,478}&
    \end{tabular}
    \caption{
      \textbf{Used maps with their sizes.}
      $|V|$ is shown with parentheses.
    }
    \label{fig:maps}
  \end{figure}
}
\figwide{inefficient-init_merged.pdf}{result-inefficient-init}{
  \textbf{The average progress of the refinement with \emph{inefficient} initial solutions.}
  The initial solver was \pibtc.
}
\figwide{efficient-init_merged.pdf}{result-efficient-init}{
  \textbf{The average progress of the refinement with \emph{efficient} initial solutions.}
  In \fieldname{random-32-32-20} and \fieldname{arena}, we used ECBS to obtain the initial solutions.
  The suboptimality was $1.2$ and $1.1$ respectively, which were adjusted to balance runtime and solution quality.
  For \fieldname{lak503d}, we prepared well-formed instances and used RPP.
  Note that it is difficult to get such instances with other settings because they are too dense.
  In \fieldname{lak503d}, we do not start the y-axis from one to see differences between rules; the improvements are tiny.
}

\section{Evaluation}
\label{sec:exp}
The experiments consist of six parts:
1)~comparing between the selection rules of agents with \emph{inefficient} initial solutions,
2)~comparing between the rules with \emph{efficient} initial solutions,
3)~evaluation of dependencies on \emph{different initial solvers},
4)~assessing costs compared to the optimal,
5)~comparing with other anytime MAPF solver,
6)~tests in challenging scenarios, i.e., huge fields with many agents.
We often use sum-of-costs divided by $\sum_{a_i \in A} \distone{i}$ as the solution quality;
smaller is better and minimum is one.
Even though optimal costs are hard to obtain, this score works as an upper bound of sub-optimality.

\subsection{Experimental Setup}
We carefully picked up several four-connected grids from~\cite{sturtevant2012benchmarks,stern2019def} as a graph $G$, shown in Fig.~\ref{fig:maps};
they are common in MAPF studies.
In all settings, we tried identical instances between solvers.
All instances were created by choosing randomly initial locations and destinations.

As refinement solver, we used one adapted from ICBS~\cite{boyarski2015icbs} for the following reasons.
First, CBS~\cite{sharon2015conflict} is a promising and actively-studied optimal solver;
it is however sensitive to tie-break of choosing high-level nodes resulting in pure CBS being poorly scalable.
ICBS, an extension of CBS, improves this aspect.
Though not state-of-the-art, ICBS is stable and has been used in many studies, e.g.,~\cite{felner2018adding,li2019improved,stern2019def}.
We thus considered that ICBS was sufficient as a baseline for our experiments.
Note that results heavily depend on the refinement solver, meaning that, the refine speed might become much faster with a faster refinement solver.

In each setting, we introduced the early stop by the timeout (Sec.~\ref{subsec:early-stop});
they were adjusted to appropriate values before experiments.

In the refinement rules \textit{composition} (Sec.~\ref{subsec:composition}), we sequentially used the rules
\textit{local-repair-around-goals} (\ref{subsec:local-repair}),
\textit{focusing-at-goals} (\ref{subsec:focus-at-goal}),
\textit{using-MDD} (\ref{subsec:using-mdd}),
and \textit{random} (\ref{subsec:random}) with 30 agents.
These rules were chosen according to preliminary results.
The switching is when no improvement is achieved for all agents.

Implementations of AFS~\cite{cohen2018anytime} and CBSH~\cite{li2019improved} were obtained from their respective authors;
we used them directly.
The simulator, including ICBS~\cite{boyarski2015icbs}, ECBS~\cite{barer2014suboptimal}, \hca and \whca~\cite{silver2005cooperative}, RPP~\cite{vcap2015prioritized}, PS~\cite{luna2011push}, PIBT~\cite{okumura2019priority}, was developed in C++,
and all experiments were run on a laptop with Intel Core i9 2.3GHz CPU and 16GB RAM.
The code is available on \url{https://kei18.github.io/mapf-IR}.

As a technical point, to obtain a fast and scalable sub-optimal solver with acceptable costs and success rate, we combined two solvers: PIBT and PS.
PIBT, which repeats one-timestep planning for all agents, produces solutions with acceptable costs, however, it sometime fails in one-shot MAPF.
PS solves MAPF in most cases (see~\cite{de2013push} for details) but only allows one agent to move at one timestep, resulting in terrible outcomes compared to the optimal.
Although we compress solutions from PS while preserving temporal dependencies of the solution, inspired by techniques in~\cite{honig2016multi}, they are still too inefficient.
We combine those two as follows.
First, run PIBT until timestep $\max$$_{a_j \in A} \distone{j}$, the minimum timestep needed for the solution.
If some agents are not on their goals at that timestep, taking this configuration as a new initial configuration, then obtain the rest of the solution using PS.
We call this solver \pibtc.
Since most parts of planning are computed by PIBT, we can expect much better outcomes than those of PS.

\subsection{with Inefficient Initial Solutions}
The first experiment aims at assessing how each rule refines inefficient initial solutions.
The initial solver was \pibtc.
The refinements were stopped after \SI{90}{\second} in the small fields (\fieldname{random-32-32-20} and \fieldname{arena}) and \SI{10}{\minute} in \fieldname{lak503d}.
The numbers of agents were fixed to $110$, $300$, and $500$, respectively.
This duration includes the time required for the initial solver.
The refinement timeout was \SI{1}{\second} for \fieldname{lak503d}, otherwise \SI{500}{\ms}.

Fig.~\ref{fig:result-inefficient-init} shows the average progress of the refinement over $25$ instances.
The rules \textit{single-agent} and \textit{local-repair-around-goals} reduce costs immediately but soon reach their limits, i.e., no improvement even with room for refinement.
The rule \textit{focusing-at-goals} dramatically improves solution quality in each case while the rule \textit{use-bottleneck-agent} does not work well as expected.
Note that \pibtc returned solutions within \SI{500}{\ms} even for the worst case (in \fieldname{lak503d} with 500 agents; see also Table~\ref{table:result-diff-init}).

\subsection{with Efficient Initial Solutions}
Next, we tested the refinement with already efficient solutions to some extent, obtained by ECBS or RPP.
The used settings were the same as the previous.

Fig.~\ref{fig:result-efficient-init} shows the results, which reveals a limitation of the rule \textit{focusing-at-goals} in \fieldname{arena} and \fieldname{random-32-32-20};
it is difficult to refine efficient enough solutions by this rule.
Rather, the rules \textit{using-MDD} and \textit{random} achieve smaller final costs.
In \fieldname{lak503d}, we often obtained initial solutions with little room for refinement, and the effect of refinement is subtle (see y-axis).
Even so, the several rules still improve the solution quality.

Throughout two experiments so far, the rule \textit{composition} successfully reduced costs with reasonable speeds;
we use this rule hereinafter.

\subsection{with Different Initial Solvers}
The third experiment evaluates dependencies to initial solvers.
We used five initial solvers: \pibtc, \hca, \whca, ECBS, and RPP.
The timeout was set \SI{1}{\second} for \fieldname{lak503d}, otherwise \SI{500}{\ms}.

Fig.~\ref{fig:result-diff-init} shows the average progress.
Table~\ref{table:result-diff-init} summarizes the details;
``cost'' is sum-of-costs divided by the lower bound.
We show both initial and last scores.
``\#success'' is the number of successful instances in $25$ instances.
Some solvers failed in some instances because they return failure due to incompleteness, or, fail to obtain solutions before the deadlines (\SI{90}{\second} in \fieldname{random-64-64} and \fieldname{lak307d}; \SI{10}{\minute} in \fieldname{lak503d}).
``runtime'' is when initial solvers return solutions.
All scores are average over success instances in all initial solvers except for \fieldname{lak503d} where \whca failed most cases; we show the average scores without \whca for \fieldname{lak503d}.

The main observation is that, although the initial costs are widely different between the solvers, the final costs end up not so.
This implies that any initial solvers can be used if you have enough time for the refinement.
In the following, we use \pibtc because it instantly returns a feasible solution and meets well with anytime property.

\subsection{v.s. Optimal Solutions}
According to Proposition~\ref{prop:no-bound}, the approximation ratio of refined solutions is unbounded from the optimal.
In practice, however, the estimation from empirical data is useful hence we evaluate this.
We used two small settings (30~agents in \fieldname{random-32-32-20}; 50~agents in \fieldname{random-32-32-10}) because optimal solvers often fail to obtain solutions within a reasonable time in large fields or with many agents.
Optimal solutions were obtained by CBSH.
The refinements continued for \SI{1}{\second} including the time for the initial solver.

\figwide{diff-init_merged.pdf}{result-diff-init}{
  \textbf{The average progress of the refinement with \emph{different initial solvers}.}
  All instances were well-formed.
  The averages are for succcess instances in all initial solvers.
  In \fieldname{lak503d}, the scores were calculated removing those of \whca because \whca failed most cases.
  From the left, the suboptimality of ECBS was $1.05$, $1.05$, and $1.1$.
  The window size of \whca was $30$, $10$, $30$.
  Those parameters were adjusted to balance success rate, cost, and runtime.
}
{
  \setlength{\tabcolsep}{1mm}
  \begin{table}
    \centering
    \caption{The detailed results with different initial solvers.}
    \label{table:result-diff-init}
    \vspace{-0.2cm}
    \scriptsize
    \begin{tabular}{rrrrrrr}
      \toprule
      && \pibtc & \hca & \whca & ECBS & RPP
      \\ \midrule
      & cost (init) & 1.219 & 1.069 & 1.096 & 1.037 & 1.035
      \\
      \fieldname{random-64-64-20} & cost (last) & 1.015 & 1.015 & 1.015 & 1.014 & 1.016
      \\
      300 agents & \#success & 25 & 23 & 15 & 25 & 25
      \\
      & runtime (ms) & 43 & 201 & 228 & 3436 & 153
      \\ \midrule
      & cost (init) & 1.190 & 1.021 & 1.155 & 1.007 & 1.007
      \\
      \fieldname{lak307d} & cost (last) & 1.003 & 1.003 & 1.003 & 1.003 & 1.003
      \\
      300 agents & \#success & 25 & 25 & 23 & 25 & 25
      \\
      & runtime (ms) & 43 & 171 & 247 & 2306 & 119
      \\ \midrule
      & cost (init) & 1.148 & 1.021 & - & 1.026 & 1.019
      \\
      \fieldname{lak503d} & cost (last) & 1.019 & 1.018 & - & 1.018 & 1.019
      \\
      500 agents & \#success & 25 & 25 & 1 & 24 & 25
      \\
      & runtime (ms) & 376 & 5716 & - & 54433 & 5640
      \\ \bottomrule
    \end{tabular}
  \end{table}
}

\fighalf{compare_optimal_merged.pdf}{result-compare-optimal}{%
  \textbf{The results v.s. \emph{optimal} solutions.}
  We show the suboptimality of 50 instances with initial scores, \SI{0.1}{\second} later, and at \SI{1}{\second}.
  The scores at \SI{1}{\second} is hard to recognize because most of them reach the optimal.
}

Fig.~\ref{fig:result-compare-optimal} summarizes the results of $50$ instances, showing the sum-of-costs divided by the optimal costs.
The average runtime of CBSH was \SI{710}{\ms} with $30$~agents and \SI{1743}{\ms} with $50$~agents while the refinement (\pibtc) got initial solutions less than \SI{3}{\ms} in all instances.
Despite large gaps between the initial and optimal costs, the refinement dramatically reduces the gaps within a short time.
Furthermore, most solutions reach the optimal within \SI{1}{\second}.

\subsection{v.s. Other Anytime MAPF Solver}
We next compared the proposal with another anytime MAPF solver, AFS, using \fieldname{random-32-32-20} while changing the number of agents.
Note that AFS theoretically converges the optimal someday but the proposal may not.
The refinement timeout was \SI{100}{\ms}.
We run both algorithms for \SI{30}{\second}.

\fighalf{compare_anytime_merged.pdf}{result-compare-anytime}{%
  \textbf{The results v.s. another \emph{anytime} MAPF solver.}
  We omit scores after \SI{10}{\second} because they are almost flat.
  The improvements of AFS are subtle and hard to recognize.
}

Fig.~\ref{fig:result-compare-anytime} shows the results of $25$ instances.
AFS failed to obtain solutions within the time for 2 instances with 90 agents.
Clearly, the proposal has an advantage;
it obtains initial solutions immediately while AFS does not, and the convergence is fast enough with better costs.

\subsection{Challenging Scenarios}
\label{subsec:challenging}
Finally, we tested the refinement with many agents in huge grids, namely, $1500$~agents in \fieldname{brc202d} and $3000$~agents in \fieldname{ost000a}.
The refinement timeout was \SI{3}{\second} for \fieldname{brc202d} and \SI{10}{\second} for \fieldname{ost000a}.
In such scenarios, it is unrealistic to obtain the optimal solutions;
hence, the iterative refinement is attractive to obtain good enough solutions.

\fighalf{challenging_merged.pdf}{result-challenging}{%
  \textbf{The results of challenging scenarios.}
}

Fig.~\ref{fig:result-challenging} shows the progress of $10$~instances with one-hour refinement.
The initial solutions were obtained \SI{4}{\second} for \fieldname{brc202d} and \SI{17}{\second} for \fieldname{ost000a} on average.%
The refinement gradually reduces costs; however, the speed of the refinement is not so fast.

\section{Discussion and Conclusion}
\label{sec:discussion}
This paper presented the iterative refinement of path finding for multiple robots.
The proposal uses two MAPF solvers as sub-procedures: a sub-optimal solver to obtain an initial solution and an optimal solver to refine the solution.
Although the framework does not guarantee to find the optimal solution, the empirical results demonstrate its usefulness, i.e., the framework finds a solution with acceptable costs in a small computation time with high scalability.
Furthermore, it is anytime planning; a desired property for real-time systems with severe deadlines.

According to the experiments, the cost is reduced to near-optimal regardless of initial solutions; however, it is better to start with efficient enough solutions if available, because we can get better solutions at an early stage.
Thus, a practical anytime MAPF scheme will be the following.
First, in parallel, start several initial solvers with different portfolios between runtime and solution quality (e.g., \pibtc and RPP).
Then, apply refinement to the first solution you get.
If another initial solver gets a better solution compared to the refined solution at that time, replace the current one with the new one.
This scheme complements a time lag of an efficient initial solver by a fast inefficient initial solver.

As future directions, we describe the following two:
1) In the experiments, the rule \textit{composition}, combining several rules with different features, was successful.
The rule \textit{composition} itself is reasonable, however, since its components were chosen empirically, an automatic selection of such rules depending on situations might be promising.
2) In challenging scenarios (Sec.~\ref{subsec:challenging}), the refinement happens gradually; not so fast.
Developing appropriate rules to achieve fast refinement for such scenarios is remaining.

MAPF studies are very active and our proposal can be better with their developments.
In particular, we used ICBS in our experiments as the refinement solver but many studies enhancing CBS, e.g.,~\cite{felner2018adding,li2019symmetry,zhang2020multi}, or, there are other promising optimal solvers~\cite{sharon2013increasing,wagner2015subdimensional,lam2019branch}.
This is the same for sub-optimal solvers as the initial solver, e.g.,~\cite{ma2019searching,okumura2019winpibt,han2020ddm,li2021eecbs}.
With their effective use, we expect that the framework becomes better than presented here.

\section*{Acknowledgment}
We are grateful to Fran\c{c}ois Bonnet for his comments on the initial manuscript.
We would like to thank Liron Cohen and Jiaoyang Li for sharing with us their implementation of AFS and CBSH, respectively.
This work was partly supported by JSPS KAKENHI Grant Numbers~20J23011, 21K11748, and 21H03423.
Keisuke Okumura thanks the support of the Yoshida Scholarship Foundation.

\bibliographystyle{IEEEtran}
\bibliography{ref}

\end{document}